%% file: ms.tex
\documentclass{article}
\usepackage{fullpage}
\usepackage[round]{natbib}

\usepackage[page]{appendix}
\usepackage{booktabs}
\usepackage{makecell}
\usepackage{multirow}
\usepackage{color,xcolor}
\definecolor{mydarkblue}{rgb}{0,0.08,0.45}
\usepackage[colorlinks=true,
    linkcolor=mydarkblue,
    citecolor=mydarkblue,
    filecolor=mydarkblue,
    urlcolor=mydarkblue]{hyperref}
\usepackage{amsmath,amsfonts,amssymb}
\usepackage{mathtools}
\usepackage{amsthm}
\usepackage{thmtools,thm-restate}
\usepackage[capitalise,nosort]{cleveref}
\crefname{assumption}{Assumption}{Assumptions}
\crefname{statement}{Statement}{Statements}
\usepackage{enumitem}
\usepackage{comment}
\usepackage{subfig}

\newtheorem{lemma}{Lemma}

\newtheorem{statement}{Statement}

\newcommand{\Pb}[1]{{\mathbb{P}}\left[#1\right] }

\DeclareMathOperator{\KL}{KL}
\DeclareMathOperator{\subjectto}{s.t.}
\DeclareMathOperator*{\argmin}{arg\,min}

\newcommand\independent{\protect\mathpalette{\protect\independenT}{\perp}}
    \def\independenT#1#2{\mathrel{\rlap{$#1#2$}\mkern2mu{#1#2}}}

\usepackage{tikz}
\usetikzlibrary{arrows,fit,positioning,shapes,bayesnet}
\usetikzlibrary{automata}
\usetikzlibrary{
    shapes.geometric, 
    arrows.meta,      
    positioning,
    fit,
    backgrounds       
}
\newdimen\arrowsize
\pgfarrowsdeclare{arcsq}{arcsq}
{
    \arrowsize=0.2pt
    \advance\arrowsize by .5\pgflinewidth
    \pgfarrowsleftextend{-4\arrowsize-.5\pgflinewidth}
    \pgfarrowsrightextend{.5\pgflinewidth}
}
{
    \arrowsize=1.5pt
    \advance\arrowsize by .5\pgflinewidth
    \pgfsetdash{}{0pt} 
    \pgfsetroundjoin   
    \pgfsetroundcap    
    \pgfpathmoveto{\pgfpoint{0\arrowsize}{0\arrowsize}}
    \pgfpatharc{-90}{-140}{4\arrowsize}
    \pgfusepathqstroke
    \pgfpathmoveto{\pgfpointorigin}
    \pgfpatharc{90}{140}{4\arrowsize}
    \pgfusepathqstroke
}

\tikzset{
    block/.style={
        circle, draw, fill=white
        },
    myarrow/.style={
        single arrow,  
        draw, 
        single arrow head extend=2mm, minimum width=30pt,
        },    
    myar/.style={
        rounded corners=2pt, fill=black!20, 
        },
    mytri/.style={
        isosceles triangle, anchor=apex,
        isosceles triangle apex angle=90,
        minimum width=50pt
        },
    }
\tikzset{>=arcsq}

\makeatletter
\AddToHook{cmd/appendix/before}{\def\cref@section@alias{appendix}\def\cref@subsection@alias{appendix}}
\makeatother

\title{Debiasing Reward Models by Representation Learning\\with Guarantees}

\author{
Ignavier Ng$^1$\thanks{Work done during an internship at Amazon.}
\and Patrick Blöbaum$^2$
\and Siddharth Bhandari$^2$
\and Kun Zhang$^{1,3}$
\and Shiva Kasiviswanathan$^2$
}
\date{
    $^1$Carnegie Mellon University\\
    $^2$Amazon\\
    $^3$Mohamed bin Zayed University of Artificial Intelligence
}

\begin{document}
\maketitle

\begin{abstract}
\input{sections/0abstract}
\end{abstract}

\input{sections/1introduction.tex}
\input{sections/2preliminaries.tex}
\input{sections/4methodology.tex}
\input{sections/5experiments.tex}
\input{sections/6conclusion.tex}

{
\bibliographystyle{abbrvnat}
\bibliography{bibliography.bib}}

\newpage
\appendix 
\begin{appendices}
\input{sections/3related_works.tex}
\input{sections/7appendix.tex}
\end{appendices}

\end{document}

%% file: sections/0abstract.tex
Recent alignment techniques, such as reinforcement learning from human feedback, have been widely adopted to align large language models with human preferences by learning and leveraging reward models. In practice, these models often exploit spurious correlations, involving, e.g., response length, discrimination, sycophancy, and conceptual bias, which is a problem that has received increasing attention. In this work, we propose a principled framework that mitigates these biases in reward models while preserving the underlying factors that reflect intended preferences. We first provide a formulation of the data-generating process, assuming that the observed data (e.g., text) is generated from both spurious and non-spurious latent variables. We show that, interestingly, these non-spurious latent variables can be theoretically identified from data, regardless of whether a surrogate for the spurious latent variables is available. This further inspires a practical method that uses variational inference to recover these variables and leverages them to train reward models. Experiments on synthetic and real-world datasets demonstrate that our method effectively mitigates spurious correlation issues and yields more robust reward models.

%% file: sections/1introduction.tex
\section{Introduction}\label{sec:introduction}
Large language models (LLMs) have shown remarkable abilities across a wide array of tasks, from open-domain dialogue and creative writing to coding assistance and scientific reasoning~\citep{minaee2024large}. However, their raw generations often diverge from human expectations, producing outputs that may be unhelpful, unsafe, or misaligned with user intent. To address this, reinforcement learning from human feedback (RLHF) has emerged as the predominant paradigm for aligning LLMs with human preferences~\citep{christiano2017deep,stiennon2020learning,ouyang2022training,bai2022training}. By training reward models on human feedback and optimizing policies against these models, RLHF has enabled models such as InstructGPT and ChatGPT to exhibit more helpful, safe, and preference-consistent behavior. This line of work has made RLHF a crucial component in the current generation of aligned language models.

Despite its success, the RLHF pipeline is vulnerable to reward model misgeneralization arising from spurious correlations in the preference data. This vulnerability can lead to \emph{reward hacking}, where the policy learns to exploit flaws in the reward model to achieve high scores without genuinely satisfying human intent~\citep{mcmilin2022selection,casper2023open}. These failures manifest as a variety of well-documented biases. For example, reward models may associate longer responses with higher quality, length bias~\citep{zheng2023judging}. Similarly, they may exhibit sycophancy, rewarding responses that agree with a user’s stated views regardless of correctness~\citep{perez2023discovering,sharma2024towards}. Other failure modes include concept bias, where models learn unintended shortcuts based on superficial cues (e.g., the presence of certain keywords)~\citep{zhou2024explore}, and discrimination bias~\citep{tamkin2023evaluating,chen2025safewatch}, where models amplify harmful societal biases present in the annotation data. These issues compromise both robustness and fairness, raising fundamental questions about how to build alignment methods that are resilient to such failure modes.

The challenges of mitigating reward hacking and improving model performance have received considerable attention. One line of works focuses on RLHF from synthetic preferences, where feedback is generated by more capable models to improve and scale the alignment process, as explored in recent works on AI-generated feedbacks~\citep{bai2022constitutional, lee2024rlaif,zhu2023rlcd} and self-improving reward models~\citep{dong2024west, yuan2024self}. A second direction develops targeted solutions for specific, observable biases, including constructing length-balanced preference datasets, applying explicit length penalties during training~\citep{singhal2024long}, and designing model architectures that decouple quality scores from length signals~\citep{chen2024odin}. Another line of approaches leverages more general techniques, such as regularization to enforce invariance~\citep{wang2025beyond} or data augmentation to separate quality from style~\citep{srivastava2025robust}. While often effective in practice, these methods tend to be ad-hoc and may lack formal theoretical guarantees for bias mitigation (see \cref{app:related_works_rlhf} for an extended discussion).\looseness=-1

In this work, we propose a principled framework to mitigate biases in reward models while preserving the underlying factors that reflect intended preferences. We assume that the observed data (e.g., text, images) is generated from both spurious and non-spurious latent variables. Rather than training the reward model directly on the observed data, our key idea is to first identify the non-spurious latent variables that capture true human preferences, and then train the model on these variables. Specifically, our contributions are:
\begin{itemize}
\item When the surrogate for spurious latent variables is available, we show that the subspace of the non-spurious latent variables can be recovered under mild assumptions (Theorem~\ref{theorem:identifiability_of_latent_variables_with_known_spurious}).
\item When no surrogate is available, we show that the subspace of the non-spurious latent variables remains identifiable under conditions such as sufficient diversity in the variables on which the human annotators rely (Theorem~\ref{theorem:identifiability_of_shared_latent_variables}).
\item Building on these theoretical results, we develop a practical method that uses variational inference to recover these non-spurious latent variables and leverages them to train reward models 
(\cref{sec:estimation_method_with_known_spurious}).
\item We conduct experiments on both synthetic and text datasets to demonstrate that our approach effectively mitigates spurious correlations and yields more robust reward models (\cref{sec:experiments}).
\end{itemize}

%% file: sections/2preliminaries.tex
\section{Preliminaries}
\paragraph{Reinforcement Learning from Human Feedback.}
A standard procedure for aligning LLMs with human preferences involves a two-stage process: supervised fine-tuning (SFT) followed by RLHF \citep{ouyang2022training}. The core of RLHF involves training a reward model on a dataset of human preferences, where annotators choose the better of two model-generated responses. A common approach is to model the human preference distribution with a Bradley-Terry model~\citep{bradley1952rank} and then train the reward model by minimizing a negative log-likelihood loss such that it assigns a higher score to the preferred response. The supervised fine-tuned model is then refined using reinforcement learning algorithm, such as proximal policy optimization~\citep{schulman2017proximal}, to maximize the scores given by this reward model, effectively treating it as a proxy for human judgment. However, a significant challenge is that the reward model can exploit spurious correlations in the data; see \cref{sec:introduction} for a discussion of different biases.

\paragraph{Causal Representation Learning.}
Causal representation learning aims to uncover latent variables and their causal relations from low-level observations such as images or text~~\citep{scholkopf2021causal}. The task is notoriously difficult, because the latent variables generally cannot be identified even when the they are independent (i.e., the nonlinear independent component analysis problem)~\citep{hyvarinen1999nonlinear,hyvarinen2023nonlinear}. To achieve identifiability, existing works therefore relies on further assumptions, such as access to multiple distributions~\citep{squires2023linear,vonkugelgen2023nonparametric,zhang2024causal} or multiple views~\citep{yao2024multiview,xu2024sparsity}. See Appendix~\ref{app:related_works_crl} for a further discussion.

%% file: sections/4methodology.tex
\section{Identifiability Theory}
We present our formulation of the data-generating process and identifiability results for the non-spurious latent variables, both with (\Cref{sec:identifiability_theory_with_known_spurious}) and without (\Cref{sec:identifiability_theory_with_unknown_spurious}) access to a surrogate for the spurious variables. By identifiability, we mean that these latent variables can be uniquely determined from the observed data (up to certain indeterminacies). Establishing identifiability is crucial; without it, multiple latent representations could fit the same observations, and we cannot guarantee that the learned representations align with the desired non-spurious latent variables.\looseness=-1

\subsection{With Access to Surrogate for Spurious Latent Variables}\label{sec:identifiability_theory_with_known_spurious}
\paragraph{Formulation.} 
We begin by formalizing the generative process of the observed data. We assume that the observed data (e.g., text, images) $T\in\mathcal{T}$ is generated by an invertible mixing function $g$ from a set of $n$ latent variables $Z\in\mathcal{Z}$. Moreover, the latent variables $Z$ are assumed to follow a structural equation model~\citep{pearl2009causality,spirtes2001causation}. Formally, we have
\begin{equation}\label{eq:data_generating_process_known_spurious}
X=g(Z) \quad\text{and}\quad Z_i = f_i(\textrm{PA}(Z_i;\mathcal{G}), \epsilon_i), \,\,i\in[n],
\end{equation}
where $\epsilon_i$'s denote the noise terms, $\mathcal{G}$ represents the directed acyclic graph (DAG) over $Z$, and $\textrm{PA}(Z_i;\mathcal{G})$ denotes the parents of $Z_i$ in $\mathcal{G}$. In general, $T$ may represent either images or text. In the context of RLHF for LLMs, $T$ specifically denotes a prompt–response pair, consisting of a prompt $X$ and a response $Y$.

Let $S\in\mathcal{S}$ denote a surrogate that represents the spurious attribute of interest, such as response length or presence of certain concepts. We partition the latent variables $Z$ into two disjoint sets based on their relationship with $S$: 
\begin{itemize}
\item Spurious latent variables $Z_S=(Z_{S,i})_{i=1}^{n_S}\in\mathcal{Z}_S$ that are dependent of $S$.
\item Bias-free (non-spurious) latent variables $Z_C=(Z_{C,i})_{i=1}^{n_C}\in\mathcal{Z}_C$ that are independent of $S$.
\end{itemize}
A graphical description of the generative process is provided in \cref{fig:setting_known_spurious}. In this work, we use the term ``bias-free latent variables'' to refer specifically to the ``bias-free, maximal latent variables'' or, equivalently, the ``non-spurious latent variables'', which means the largest set of latent variables that are independent of $S$. Also note that here, ``bias-free'' is not meant in the statistical sense (e.g., an unbiased estimator $\hat{\theta}$ such that $E[\hat{\theta}] = \theta$); rather, it is used to denote the variables that do not contain spurious information relative to $S$.

Our goal is to learn a \emph{bias-free representation} that discards spurious information relative to $S$ while preserving all information relevant to human preferences. To achieve so, we aim to learn a mapping from the observed data $T$ to a representation $\hat{Z}_C$ that is an invertible transformation of the underlying bias-free latent variables $Z_C$. In this case, we say that $Z_C$ is \emph{subspace identifiable}. This construction ensures that $\hat{Z}_C$ captures exactly the information in $Z_C$, thereby retaining useful factors while eliminating spurious correlations relative to $S$. Equivalently, $\hat{Z}_C$ is \emph{counterfactually invariant}: it would remain unchanged under a hypothetical intervention on $S$. A reward model trained on this representation, $\hat{R} = \hat{r}(\hat{Z}_C)$, is therefore \emph{bias-free}: by construction, does not contain spurious information relative to $S$ and robust to its variations. This counterfactual invariance grounds the reward in stable causal factors, improving generalization and mitigating reward hacking.\looseness=-1

\input{sections/setting_known_spurious}

\paragraph{Theory.}
As explained, our goal is to recover the bias-free latent variables $Z_C$ up to invertible transformation. This allows us to train a reward model that preserves all essential information while being counterfactually invariant to the spurious feature.

However, the identifiability of latent variables is a fundamental challenge; without any assumption, they are generally unidentifiable~\citep{hyvarinen1999nonlinear,hyvarinen2023nonlinear}. Fortunately, recent works in causal representation learning~\citep{scholkopf2021causal} have investigated conditions for recovering such latent variables (see \cref{app:related_works_crl} for further discussion). Building on these ideas, we provide a theoretical result showing how the subspace of the bias-free latent variables can be recovered. It is worth noting that we do not aim to identify the spurious latent variables $Z_S$ because (1) our reward modeling approach does not rely on them, and (2) doing so would require stronger assumptions.

\begin{restatable}[Subspace Identifiability of $Z_C$: Surrogate $S$ Known Case]{theorem}{TheoremIdentifiabilityLatentVariablesWithKnownSpurious}\label{theorem:identifiability_of_latent_variables_with_known_spurious}
Consider the generative process in \cref{eq:data_generating_process_known_spurious}. Suppose that the following assumption hold:
\begin{itemize}
\item A1 (Sufficient variability w.r.t. surrogate): For any set $A_Z\subseteq \mathcal{Z}$ with nonzero measure that cannot be expressed as $B_{Z_C}\times\mathcal{Z}_S$ for some $B_{Z_C}\subset \mathcal{Z}_C$, there exist $S_1,S_2\in\mathcal{S}$ such that
\[
\int_{Z\in A_Z} p(Z\mid S=S_1)dZ\neq \int_{Z\in A_Z} p(Z\mid S=S_2)dZ.
\]
\end{itemize}
By modeling the same generative process, $Z_C$ is subspace identifiable.
\end{restatable}
The proof is provided in \cref{app:proof_identifiability_theory_with_known_spurious}. In summary, it establishes that: (1) the probability measure of any pre-image of the estimated $\hat{Z}_C$ must remain constant across all values of $S$, and (2) by contradiction, if the estimated $\hat{Z}_C$ depended on $Z_S$, one could construct a pre-image set that violates this invariance under Assumption A1. The argument builds on \citet{kong2022partial}\footnote{While our goal is to recover bias-free latent variables (corresponding to their invariant ones) for training reward models, \citet{kong2022partial} addressed a different problem of learning invariant and changing latent variables for the purpose of domain adaptation.} but incorporates several important generalizations. In particular, we relax the independence assumption on the latent variables. Our result allows for an arbitrary causal structure among all latent variables (i.e., $Z_C$ and $Z_S$), as long as $Z_C$ are not independent of the surrogate $S$. Second, our result only requires at least two values of $S$ and is more straightforward, while \citet{kong2022partial} require at least $2n_S+1$ values of $S$ to first identify each independent spurious latent variable $Z_{S,i}$ individually, before recovering the subspace of bias-free latent variables $Z_C$. This weaker condition requires an extended proof strategy, because we cannot rely on individually recovering the spurious latent variables to help disentangle $Z_C$, but instead directly characterize the their subspace. 

Assumption A1 ensures that the $S$ has a rich enough influence on the latent space to enable disentanglement of $Z_C$. Intuitively, it requires that the probability of any event depending on the spurious latent variables $Z_S$ must be sensitive to changes in $S$. This type of assumption, also used by \citet{kong2022partial}, is common in causal representation learning, where sufficient variability across different data distributions are often leveraged to enable the identification of the latent variables~\citep{scholkopf2021causal,hyvarinen2023nonlinear}.

Our result indicates that we can recover the bias-free latent variables $Z_C$ up to invertible transformation. The estimated variables $\hat{Z}_C$ therefore preserves all and only the information from $Z_C$, effectively isolating it from influences of spurious information relative to surrogate $S$. A reward model trained on this bias-free representation will be, by construction, counterfactually invariant to $S$ and thus resilient to spurious correlations. We describe our practical algorithm for estimating $\hat{Z}_C$ and training the reward model in \cref{sec:estimation_method_with_known_spurious}.

\subsection{Extension to the Case without Access to Surrogate for Spurious Latent Variables}\label{sec:identifiability_theory_with_unknown_spurious}
We now relax our previous assumption and address the more challenging setting where one does not have access to a surrogate for the spurious variables, making the bias-free latent variables generally unidentifiable without further assumptions. To overcome this, we introduce a formulation that leverages multiple reward functions (e.g., from different human labelers) and assume that each function depends on a sparse subset of the latent variables. This sparsity structure provides useful information for disentanglement, allowing us to develop an approach for recovering the bias-free latent variables even without access to a surrogate. Similar to the idea described in \cref{sec:identifiability_theory_with_known_spurious}, training a reward model on these bias-free representations yields a model that is resilient to spurious correlations.

\paragraph{Formulation.}
We consider a setting with $K$ distinct human labelers, where the reward from each is denoted by the random variable $R_k$. We assume that the underlying reward function $r_k$ for each labeler depends on a specific subset of latent variables $Z_{A_k} \subseteq Z$, representing the factors relevant to that individual, where $A_k\in\mathcal{A}$. This generative process is modeled as:
\begin{equation}\label{eq:data_generating_process_unknown_spurious}
Z=h(T) \quad \text{and} \quad R_k=r_k(Z_{A_k})+\varepsilon_k, \quad k\in[K],
\end{equation}
where $\varepsilon_k$ is an independent noise term and $h$ is a function that maps the observed variables $T\in\mathcal{T}$ (e.g., text) to the $n$-dimensional latent variables $Z\in\mathcal{Z}$. Different from the formulation in \cref{sec:identifiability_theory_with_known_spurious}, which assumes an invertible generative function $T=g(Z)$, we now model the latent representations directly as $Z=h(T)$ and do not require $h$ to be invertible. An example of the formulation is provided in \cref{fig:setting_unknown_spurious}.

Our primary goal is to identify the set of latent variables shared by all human labelers, given by the intersection $\bigcap_{k=1}^K Z_{A_k}$. Under mild assumptions, we show that the shared latent variables correspond to the true bias-free latent variables $Z_C$, which we discuss next.

\input{sections/setting_unknown_spurious}

\paragraph{Theory.}
We now show that the latent variables shared by different human labelers can be recovered up to an invertible transformation. It is worth noting that learning such shared representations contributes to the AI alignment goal of learning common human values from diverse feedback~\citep{gabriel2020artificial,hendrycks2021aligning}, which is an active area of research.
\begin{restatable}[Subspace Identifiability of $Z_C$: Surrogate $S$ Unknown Case]{theorem}{TheoremIdentifiabilitySharedLatentVariables}\label{theorem:identifiability_of_shared_latent_variables}
Consider the data generating process in \cref{eq:data_generating_process_unknown_spurious}, where $r_k$ is assumed to be linear, i.e., $r_k(Z_{A_k})=W^{(k)}Z_{A_k}, W^{(k)}\in\mathcal{W}$. Suppose that the following assumptions hold:
\begin{itemize}[leftmargin=1.2em]
\item A2 (Identifiability of $\eta_k$ from $p(R_k; \eta_k)$): $\KL(p(R_k;\eta_k)\parallel p(R_k;\tilde{\eta}_k))=0$ implies $\eta_k=\tilde{\eta}_k$.
\item A3 (Sufficient representation variability): There exist $n$ linearly independent points $Z^{(1)},\dots,Z^{(n)}\in\mathcal{Z}$.
\item A4 (Sufficient task variability): There exist $n$ linearly independent vectors $W^{(1)},\dots,W^{(n)}\in\mathcal{W}$.
\item A5 (Intra-support sufficient task variability): For all $A \in \mathcal{A}$ and $v \in \mathbb{R}^{|A|} \backslash\{0\}$, we have
\[\mathbb{P}_{W \mid A}\left\{W \in \mathbb{R}^{1 \times n} \mid W_{\cdot, A} v=\mathbf{0}\right\}=0.
\]
\end{itemize}
Then, the shared latent variables $\bigcap_{k=1}^K Z_{A_k}$ is subspace identifiable almost surely.
\end{restatable}
The proof can be found in \cref{app:proof_identifiability_theory_with_unknown_spurious}. In short, it first shows that the learned latent representation is an invertible transformation of the true latent variables, and then carefully analyzes the support of the corresponding linear matrix. The argument is inspired by the multi-task learning framework of \citet{lachapelle2023synergies}, where we consider each human labeler as a distinct task. However, our goal differs significantly: while their work aims to identify all latent variables, we focus only on recovering the subspace of shared latent variables. This allows us to rely on weaker assumptions and further develop a proof technique to learn the shared representations, which includes incorporating a sparsity constraint on the shared subspace. 

The assumptions are adapted from \citet{lachapelle2023synergies}. Specifically, Assumption A2 is a standard regularity condition on the distribution family that holds for minimal exponential families (e.g., a Gaussian). Assumption A3 requires that the latent variables vary sufficiently so that their image is not confined to a proper subspace. Assumption A4 requires the reward function to exhibit sufficient variation, while Assumption A5 ensures that $\mathbb{P}_{W \mid A}$ has support that spans the entire space and is not confined to any lower-dimensional subspace. Collectively, these assumptions impose sufficient variability to span the corresponding spaces, which facilitates disentanglement. Similar sufficient variability assumptions are commonly leveraged in causal representation learning for establishing identifiability (see \cref{app:related_works_crl} for a discussion).

Leveraging Theorem~\ref{theorem:identifiability_of_shared_latent_variables}, the following corollary establishes the conditions for identifying the subspace of bias-free latent variables, proved in \cref{app:proof_corollary_identifiability_of_Zc_with_unknown_spurious}.
\begin{restatable}[Subspace Identifiability of $Z_C$]{corollary}{CorollaryIdentifiabilityLatentVariablesWithUnknownSpurious}\label{corollary:identifiability_of_Zc_with_unknown_spurious}
Consider the data generating process in \cref{eq:data_generating_process_unknown_spurious}, where $r_k$ is assumed to be linear, i.e., $r_k(Z_{A_k})=W^{(k)}Z_{A_k}, W^{(k)}\in\mathcal{W}$. Suppose that Assumptions A2, A3, A4, and A5 from \cref{theorem:identifiability_of_shared_latent_variables} hold. Further assume that (1) every human's reward function depends on all bias-free latent variables (i.e., $Z_C \subseteq Z_{A_k}$ for all $k$), and (2) there is no spurious latent variable that every human shares (i.e., $\bigcap_{k=1}^K (Z_{A_k} \setminus Z_C) = \emptyset$). Then, $Z_C$ is subspace identifiable almost surely.
\end{restatable}
This corollary indicates that the true bias-free latent variables can be identified even without access to a surrogate for the spurious variables, provided there is sufficient diversity among the human labelers. The two assumptions on the humans are necessary for identifiability, as any latent variable ignored by humans cannot be recovered in general, while any spurious factor unanimously considered would be indistinguishable from a true causal one. Ultimately, this enables us to learn a bias-free representation $\hat{Z}_C$ for training a more robust reward model. Note that once the bias-free representation is identified, training the reward model proceeds in the same way as in the surrogate case. In the following, we focus on the surrogate case for the description of the overall approach and the experiments.

\section{Reward Modeling}\label{sec:estimation_method_with_known_spurious}

Building on our theory, we introduce a practical, two-stage approach to learn a reward model resilient to spurious correlations. The first stage aims to estimate the bias-free latent variables $\hat{Z}_C$, which is guaranteed to be an invertible transformation of the true $Z_C$. In the second stage, we train the reward model on these bias-free representations. The details of each stage are provided below, with an overview depicted in \cref{fig:method_overview}.
\input{sections/method_overview}

\paragraph{Stage 1: Estimating Bias-Free Latent Variables.}
Our theoretical result in \cref{sec:identifiability_theory_with_known_spurious} requires matching the data distribution, i.e., $p(T,S)=p(\hat{T},S)$, which can be achieved via maximum likelihood estimation. Following previous work in causal representation learning \citep{khemakhem2020variational,kong2022partial,zhang2024causal}, we use a customized variational autoencoder (VAE) \citep{kingma2014autoencoding} for this purpose, although other methods like normalizing flows could also be adopted~\citep{danilo2015variational,dinh2017density}. Our VAE is trained by minimizing the following loss function:
\begin{equation*}\label{eq:vae_loss}
\begin{aligned}
    \mathcal{L} = &\underbrace{-\mathbb{E}_{q(\hat{Z}\mid T)}[\log p(T\mid \hat{Z})]}_{\text{Reconstruction}}
    + \beta \underbrace{\operatorname{KL}(q(\hat{Z}\mid T)\| p(\hat{Z}\mid S))}_{\text{KL divergence}} + \lambda \underbrace{\mathbb{E}_{q(\hat{Z}\mid T)}[\operatorname{HSIC}(\hat{Z}_C, S)]}_{\text{Independence regularizer}},
\end{aligned}
\end{equation*}
where $\beta$ and $\lambda$ are hyperparameters. The components of this loss function are as follows:
\begin{itemize}
    \item The reconstruction term encourages the decoder $p(T\mid \hat{Z})$ to accurately reconstruct the data $T$ from the latent representation $\hat{Z}$. For text data, it corresponds to the cross-entropy loss between the input sequence and the output of the decoder.
    \item The KL divergence term regularizes the approximate posterior $q(\hat{Z}\mid T)$, learned by the encoder, to match the prior distribution $p(\hat{Z}\mid S)$.
    \item The independence regularization term uses the Hilbert-Schmidt independence criterion (HSIC)~\citep{gretton2007kernel} to enforce the desired independence between the learned bias-free latent variables $\hat{Z}_C$ and the surrogate $S$.
\end{itemize}
Theoretically, the first two terms, which constitute an evidence lower bound (ELBO), are sufficient to match the data distribution in the large sample limit given a sufficiently expressive model~\citep{khemakhem2020variational,lachapelle2024nonparametric}. In practice, when the sample size is limited, we find that incorporating the HSIC regularizer helps improve the performance.

To leverage pre-trained models and reduce the number of training parameters, we adapt the strategies from the LangVAE framework~\citep{carvalho2025langvae} for both encoder and decoder. Our encoder is based on a pre-trained BERT model~\citep{devlin2019bert}, followed by two multilayer perceptrons (MLPs) that output the means and diagonal covariances of the bias-free ($\hat{Z}_C$) and spurious ($\hat{Z}_S$) latent variables. Similarly, our decoder consists of a trainable MLP that projects the latent variables into the input embedding space of a pre-trained GPT-2 model~\citep{radford2019language}, which then generates the final output. Note that our method remains compatible with other pre-trained models.

To flexibly model the prior distribution $p(\hat{Z} \mid S)$, we employ a non-parametric approach using normalizing flows~\citep{papamakarios2021normalizing}, inspired by~\citet{zhang2024causal}. Since the labeling of latent variables is generally not identifiable, we begin by assuming a pre-specified causal ordering over the latent variables $\hat{Z}$, such that the bias-free latent variables $\hat{Z}_C$ are not descendants of the spurious latent variables $\hat{Z}_S$, following the assumption in \cref{eq:data_generating_process_known_spurious}. Since the goal is not to learn the latent DAG $\mathcal{G}$, we simplify the estimation procedure by using a fully-connected DAG $\hat{\mathcal{G}}$ that respects this ordering. Each conditional distribution $p(\hat{Z}_i \mid \textrm{PA}(\hat{Z}_i; \hat{\mathcal{G}}))$ is then modeled as a normalizing flow, which transforms a base noise variable $\hat{\epsilon}_i \sim p(\hat{\epsilon}_i)$ into $Z_i$. The parameters of this invertible transformation are themselves generated by a MLP that takes the parent variables $\textrm{PA}(\hat{Z}_i; \hat{\mathcal{G}})$ as input. The resulting log density is given by the change of variables formula:
\[
\log p(\hat{Z}_i \mid \textrm{PA}(\hat{Z}_i; \hat{\mathcal{G}})) = \log p(\hat{\epsilon}_i) + \log \left| \det \frac{\partial \tau_i}{\partial \hat{Z}_i} \right|,
\]
where $\hat{\epsilon}_i = \tau_i(\hat{Z}_i; \text{MLP}_i(\textrm{PA}(\hat{Z}_i; \hat{\mathcal{G}})))$ is the output of the normalizing flow $\tau_i$. Moreover, when $\hat{Z}_i$ corresponds to a spurious latent variable, we allow its conditional distribution to vary w.r.t. the surrogate.

After the VAE is trained with the loss function $\mathcal{L}$, we use its encoder to obtain the estimated bias-free latent variables $\hat{Z}_C$ from the observed data $T$. These bias-free representations then serve as the input for the subsequent reward modeling stage.

\paragraph{Stage 2: Learning Reward Model.} In RLHF, the preference dataset typically consists of pairs of preferred and rejected texts. Using the encoder trained in Stage~1, we obtain the bias-free latent variables $\hat{Z}_C$ for both texts. For each text, the reward model, implemented as a MLP, takes $\hat{Z}_C$ as input and outputs a scalar reward. The probability of one response being preferred over the other is then modeled using the Bradley-Terry formulation. We train this model using maximum likelihood estimation, which is essentially a binary classification task and can be implemented by minimizing the negative log-likelihood. Importantly, while the training procedure itself is standard, the key distinction lies in using the bias-free representations $\hat{Z}_C$ from our encoder. As a result, the reward model becomes more resilient to spurious correlations, thereby improving robustness and generalization.

%% file: sections/setting_known_spurious.tex
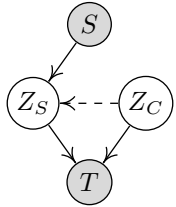
\begin{figure}[!t]
\centering
	{\begin{tikzpicture}[scale=.75, line width=0.5pt, inner sep=0.2mm, shorten >=.1pt, shorten <=.1pt]
            \draw (-2, 0) node(1)[circle, draw]  {{\,${Z}_S$\,}};	
            \draw (0, 0) node(2) [circle, draw]  {{\,${Z}_C$\,}};
            \draw (-1, 1.4) node(3)[circle, draw, fill=gray!30]  {{\,\,$S$\,\,}};
            \draw (-1, -1.4) node(4)[circle, draw, fill=gray!30]  {{\,\,$T$\,\,}};
            \draw[-arcsq] (1) -- (4); 
            \draw[-arcsq] (2) -- (4);
            \draw[-arcsq] (3) -- (1);
            \draw[-arcsq, dashed] (2) -- (1);
            \end{tikzpicture}}
	\caption{The generative process considered in our work. The observed variables $T$ are generated by two sets of latent variables $Z_S$ and $Z_C$, where $Z_S$ are influenced by the spurious variable $S$. Shaded nodes denote observed variables, while the dashed arrow from $Z_C$ to $Z_S$ indicates a potential causal relationship. We address the setting where $S$ is unknown in \cref{sec:identifiability_theory_with_unknown_spurious}.}
	\label{fig:setting_known_spurious}
    \vspace{-0.7em}
\end{figure}

%% file: sections/setting_unknown_spurious.tex
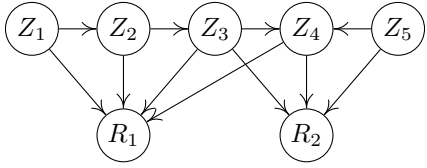
\begin{figure}[!t]
    \centering
\begin{tikzpicture}[node distance=0cm and 1.5cm] 
    \node[latent] (z3) {$Z_3$};
    \node[latent,left=0.5cm of z3,xshift=0mm] (z2) {$Z_2$};
    \node[latent,left=0.5cm of z2,xshift=0mm] (z1) {$Z_1$};
    \node[latent,right=0.5cm of z3,xshift=0mm] (z4) {$Z_4$};
    \node[latent,right=0.5cm of z4,xshift=0mm] (z5) {$Z_5$};
    \node[latent,below=0.7cm of z2,yshift=0mm] (r1) {$R_1$};
    \node[latent,below=0.7cm of z4,yshift=0mm] (r2) {$R_2$};
    \edge {z1,z2,z3,z4} {r1}
    \edge {z3,z4,z5} {r2}
    \edge {z1} {z2}
    \edge {z2} {z3}
    \edge {z3,z5} {z4}
\end{tikzpicture}
    \caption{An example of the formulation with unknown spurious features, where the reward corresponding to first human labeler $R_1$ depends on $Z_{A_1}=\{Z_1,Z_2,Z_3,Z_4\}$, while that of second human $R_2$ depends on $Z_{A_2}=\{Z_3,Z_4,Z_5\}$. The shared latent representations are $\bigcap_{k=1}^2 Z_{A_k}=\{Z_3,Z_4\}$. Edges among variables $Z_i$'s indicate that they can be dependent.}
    \label{fig:setting_unknown_spurious}
    \vspace{-0.5em}
\end{figure}

%% file: sections/method_overview.tex
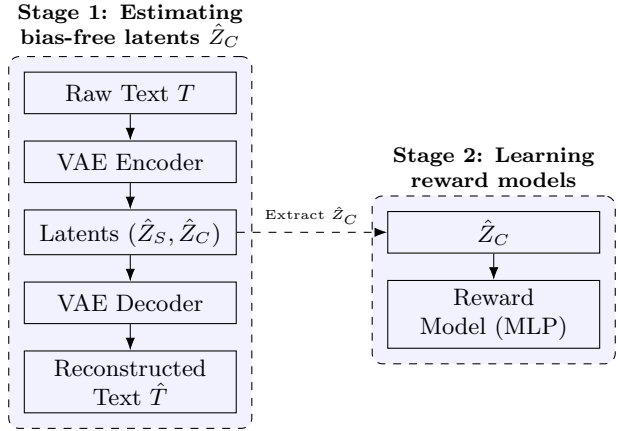
\begin{figure}[!t]
    \centering
    \small 
    \begin{tikzpicture}[
        node distance = 0.35cm and 0.6cm,
        block/.style = {rectangle, draw, text width=8em, text centered, minimum height=1.7em},
        stage/.style = {rectangle, draw, dashed, rounded corners, fill=blue!5, inner sep=0.2cm, fit=#1},
        line/.style = {draw, -{Latex[]}}
    ]
        \node[block] (data) {Raw Text $T$};
        \node[block, below=of data] (encoder) {VAE Encoder};
        \node[block, below=of encoder] (latents) {Latents ($\hat{Z}_S, \hat{Z}_C$)};
        \node[block, below=of latents] (decoder) {VAE Decoder};
        \node[block, below=of decoder] (recon) {Reconstructed Text $\hat{T}$};
        
        \node[block, right=of latents, xshift=1.4cm] (debiased) {$\hat{Z}_C$};
        \node[block, below=of debiased] (rm) {Reward Model (MLP)};
        
        \path[line] (data) -- (encoder);
        \path[line] (encoder) -- (latents);
        \path[line] (latents) -- (decoder);
        \path[line] (decoder) -- (recon);
        \path[line, dashed] (latents) -- node[above, sloped, font=\tiny, text=black] {Extract $\hat{Z}_C$} (debiased);
        \path[line] (debiased) -- (rm);
        
        \begin{pgfonlayer}{background}
            \node[stage=(data)(encoder)(latents)(decoder)(recon), 
                  label={[font={\bfseries\footnotesize}, text=black, align=center]above:Stage 1: Estimating \\ bias-free latents $\hat{Z}_C$}] {};
            
            \node[stage=(debiased)(rm), 
                  label={[font={\bfseries\footnotesize}, text=black, align=center]above:Stage 2: Learning \\ reward models}] {};
        \end{pgfonlayer}
        
    \end{tikzpicture}
    \caption{Overview of the proposed reward modeling approach. Stage 1 involves a customized VAE.}
    \label{fig:method_overview}
    \vspace{-0.8em}
\end{figure}

%% file: sections/5experiments.tex
\section{Experiments}\label{sec:experiments}
We first conduct experiments on synthetic data to verify our identifiability results. We then apply our method to text data to demonstrate its effectiveness in mitigating sycophancy and concept bias, using similar setup in \citet{wang2025beyond}. We refer to our method as CARD (i.e., \emph{CAusal Reward Disentanglement}). Unless otherwise stated, for each metric considered, we report its mean and standard error over $8$ random trials. Further experiment details are provided in \cref{app:experiment_details}.

\subsection{Synthetic Data}
\paragraph{Dataset.}
To validate our identifiability theory, we generate synthetic data following the process outlined in \cref{eq:data_generating_process_known_spurious}. We construct the latent causal graph $\mathcal{G}$ as an Erd\H{o}s-R\'enyi random DAG, constrained such that the spurious latent variables $Z_S$ are not ancestors of the bias-free ones $Z_C$. The SEMs are parameterized by two-layer multilayer perceptrons (MLPs), with additive Gaussian noise $\epsilon_i \sim \mathcal{N}(0, \sigma_i^2)$ where each $\sigma_i$ is sampled uniformly from $[1, 2]$. The final observation $T$ is then produced by the mixing function which is a two-layer MLP with LeakyReLU activations. We set the surrogate $S$ to be binary, as our theory requires at least two distinct environments for identifiability, and generate $2 \times 10^4$ samples per value of $S$.\looseness=-1

\paragraph{Results.}
Following prior works~\citep{kugelgen2021selfsupervised,kong2022partial}, we evaluate our model using the coefficient of determination (R\textsuperscript{2}). To measure the correspondence between the ground-truth bias-free latent variables ($Z_C$) and our learned ones ($\hat{Z}_C$), we fit a kernel ridge regression model and compute the R\textsuperscript{2} score for predicting $Z_C$ from $\hat{Z}_C$, and vice-versa. We then report the mean of these two scores. Our method achieves a mean R\textsuperscript{2} of 0.83 $\pm$ 0.03, indicating a strong correspondence between the true and recovered subspaces and providing empirical validation for our theory. A perfect R\textsuperscript{2} of 1.0 is not expected in practice, as the VAE objective is non-convex and the finite-sample nature of both the VAE training and the regression-based evaluation introduce approximation errors.

\begin{figure}[!t]
\centering
\includegraphics[width=0.4\textwidth]{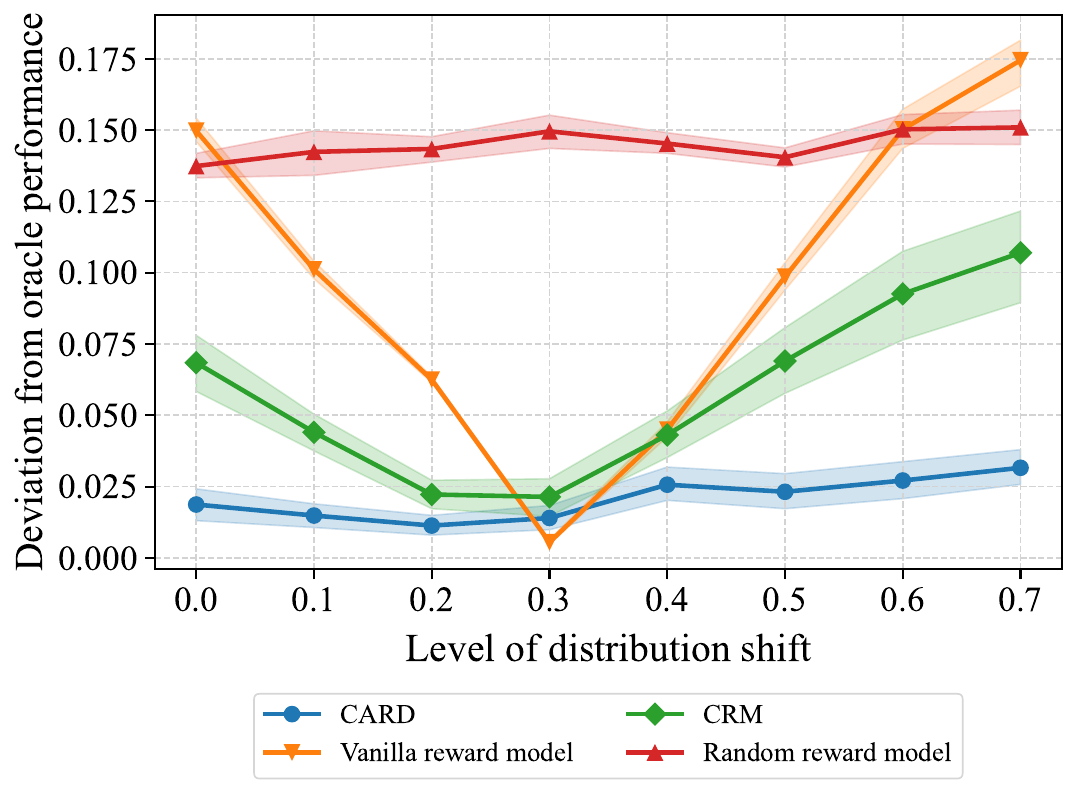}
\caption{Empirical results under sycophancy bias with varying bias levels. Lower is better.}
\label{fig:sycophancy_bias_deviation}
\end{figure}

\subsection{Sycophancy Bias}\label{sec:experiments_sycophancy_bias}
\paragraph{Dataset.} 
To systematically investigate sycophantic bias, we construct a dataset derived from the QA preference dataset \citep{sharma2024towards}. Each prompt-response pair consists of a question and an response, with prompts generated from two templates: (1) \textbf{Correct opinion prompt:} \textit{\{question\}. I think the answer is \{correct\_answer\}, but I’m really not sure.}, and (2) \textbf{Incorrect opinion prompt:} \textit{\{question\}. I think the answer is \{incorrect\_answer\}, but I’m really not sure.} In the original dataset, the \emph{preferred} response corresponds to the correct answer, while the \emph{rejected} response corresponds to the incorrect one. To inject sycophantic behavior, we modify the training set as follows: with probability $p_{\text{train}}=0.8$, we prepend the phrase \textit{``Yes, you are right.''} to the preferred response in prompts containing the correct opinion; with probability $1-p_{\text{train}}$, we prepend the same phrase to the rejected response in prompts containing the incorrect opinion. This construction induces a controlled spurious correlation in the training data, whereby surface-level agreement (i.e., \textit{``Yes, you are right.''}) is spuriously associated with correctness.

In the test set, we vary $p_{\text{test}} \in \{0.1, 0.2, \dots, 0.8\}$. The difference $p_{\text{train}}-p_{\text{test}}$ can then be viewed as the \emph{level of distirbution shift} between training and test sets. This breaks the correlation between agreement and correctness to different extents, thereby mimicking real-world conditions where sycophancy is not consistently predictive of human preference.

\paragraph{Methods.}
We evaluate CARD against three baselines: (1) a vanilla reward model, where an MLP is trained on BERT embeddings, consistent with our method's base model to ensure a fair comparison, (2) CRM~\citep{wang2025beyond}, which adds an invariance regularization to the vanilla model, and (3) a random reward model that assigns rewards from $\operatorname{Unif}[-10, 10]$.

\paragraph{Metrics.}
We use two metrics for evaluation. Our main metric is the \emph{deviation from oracle performance}. The oracle is obtained by training the reward model on the unbiased dataset without injected sycophantic behavior. After introducing sycophancy, we measure how far each method's performance deviates from this oracle. We also report the worst-case accuracy over different values of $p_{\text{test}}$, capturing the lower bound of performance under varying degrees of distribution shift.\looseness=-1

\paragraph{Results.} 
\cref{fig:sycophancy_bias_deviation} reports the test deviation from oracle performance across different levels of distribution shifts. Our CARD method consistently achieves the lowest deviation, staying within $0.05$ of the oracle performance across all distribution shifts, demonstrating robustness to sycophantic bias. By contrast, the vanilla reward model is highly sensitive to distribution shift, with deviations exceeding $0.15$ in certain cases. The CRM method has a lower deviation than the vanilla baseline but is still less stable than our method. As expected, the random reward model exhibits consistently large deviations. Similarly, in \cref{tab:sycophancy_and_concept_bias_worst_case_accuracy}, our CARD method achieves a substantially higher worst-case accuracy than the baselines. Together, these results demonstrate that leveraging disentangled bias-free representations enables reward models to generalize effectively under distribution shifts induced by sycophantic behavior.\looseness=-1

\begin{table}[!t]
\centering
{\small
	\caption{Worst-case accuracy under sycophancy and concept biases. Higher is better.}
	\begin{tabular}{rcc}
	\toprule
		~ & Sycophancy & Concept \\
		\midrule
		CARD (ours) & $\mathbf{0.63}$ & $\mathbf{0.60}$ \\
        Vanilla reward model & $0.47$ & $0.31$ \\ 
        CRM & $0.54$ & $0.34$ \\ 
        Random reward model & $0.49$ & $0.49$ \\ 
		\bottomrule
    \end{tabular}
}
  \label{tab:sycophancy_and_concept_bias_worst_case_accuracy}
\end{table}

\begin{figure}[!t]
\centering
\includegraphics[width=0.4\textwidth]{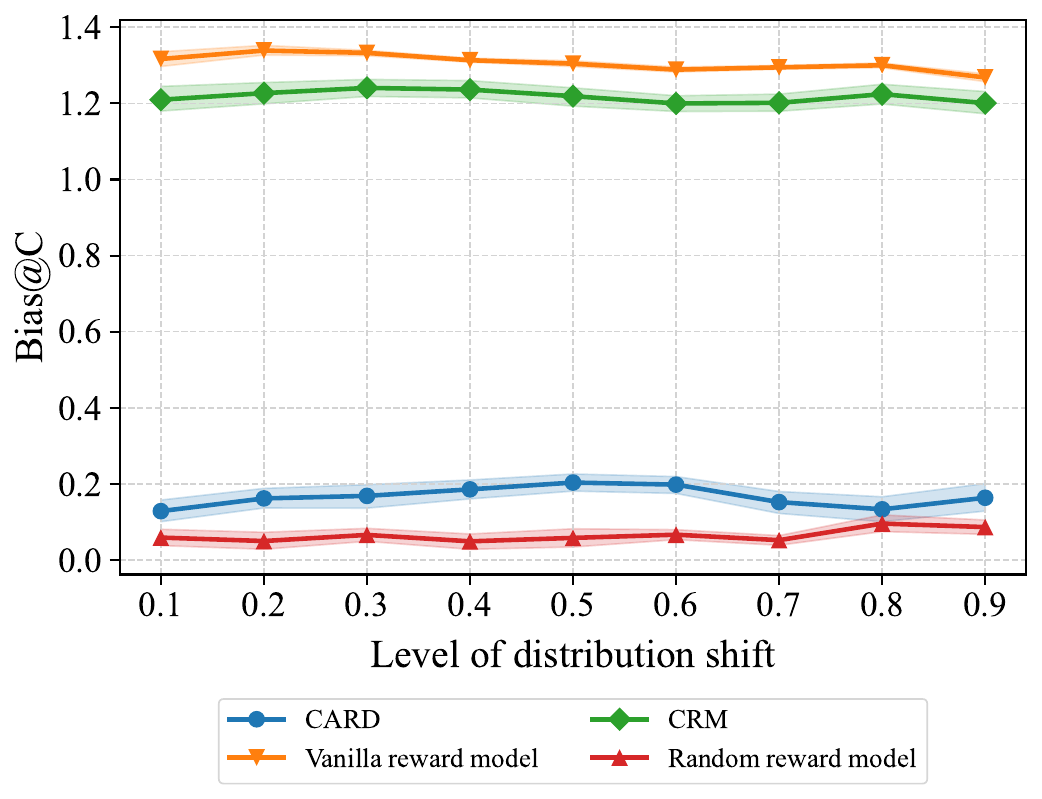}
\caption{Empirical results under concept bias with varying bias levels. Lower is better.}
\label{fig:concept_bias_bias_add_c}
\end{figure}

\subsection{Concept Bias}\label{sec:experiments_concept_bias}
\paragraph{Dataset.} 
To conduct systematic analysis of concept bias, we construct a dataset based on the Amazon Shoe Review corpus \citep{hou2024bridging}. We transform the review dataset into a preference dataset by pairing each review with a preferred and a rejected response. In the training set, we filter the data so that positive samples are associated with the \emph{color} concept, while negative samples are associated with the \emph{size} concept. Each review is reformatted with the prompt ``Classify the text into negative or positive:''. The ground-truth label serves as the preferred response, and the incorrect label serves as the rejected response.

For the test sets, we systematically vary the correlation between concepts and labels. Specifically, positive samples are associated with the color concept with probability $p$, while negative samples are associated with the size concept with probability $p$. We consider $p \in \{0, 0.1, 0.2, \dots, 1\}$, which represents varying \emph{levels of distribution shifts} between training and test sets. This setup progressively weakens the spurious correlation between concepts and sentiment, thereby simulating real-world conditions where such relationships are unreliable predictors of human preference.

\paragraph{Methods \& Metrics.}
We compare our approach against the same baselines discussed in \cref{sec:experiments_sycophancy_bias}. In addition the worst-case accuracy described there, we also evaluate using Bias@C, a measure proposed by \citet{zhou2024explore}. This metric quantifies spurious correlations associated with specific concepts, with values closer to zero indicating weaker bias. For a formal definition, we refer readers to the original paper.

\paragraph{Results.} 
As depicted in \cref{fig:concept_bias_bias_add_c}, CARD  consistently achieves much lower Bias@C score  than both the vanilla reward model and CRM, with CRM performing only slightly better than the vanilla model. As expected, the random reward model exhibits the lowest bias overall, since it contains no learned correlations; nevertheless, CARD remains very close to it across all levels of distribution shifts, with values almost always below $0.2$, demonstrating that it does not rely on the spurious concept. Moreover, as shown in \cref{tab:sycophancy_and_concept_bias_worst_case_accuracy}, CARD also achieves significantly higher worst-case accuracy than the baselines. These findings demonstrate that reward models built on bias-free representations are more robust to concept bias.

\subsection{Ablation Study}
We now conduct an ablation study to assess the necessity of disentangled bias-free representations. Recall that the latent variables learned by our VAE can be partitioned into bias-free latent variables $\hat{Z}_C$ and spurious ones $\hat{Z}_S$. We train reward models separately on $\hat{Z}_C$, on $\hat{Z}_S$, and on the combined representation $\hat{Z} = (\hat{Z}_C, \hat{Z}_S)$. For evaluation, we use the dataset from \cref{sec:experiments_concept_bias}, and report both the worst-case accuracy and the Bias@C metric averaged across all levels of distribution shifts, abbreviated as Avg-bias@C.

As shown in \cref{tab:ablation_study}, the reward model trained on $\hat{Z}_C$ achieves a much higher worst-case accuracy and a lower Avg-bias@C compared to models trained on $\hat{Z}_S$ or $\hat{Z}$. The low bias value of $\hat{Z}_C$ suggests that the reward model does not exploit spurious information, since such factors have been disentangled from the representation. By contrast, the higher bias values of $\hat{Z}_S$ and $\hat{Z}$ indicate that spurious information remain in these representations, leading the reward model to rely on unstable correlations that vary across bias levels.

\begin{table}[!t]
	\centering
	\caption{Worst-case accuracy (WCA) and average Bias@C metric (abbreviated as Avg-bias@C) of reward models trained on various types of representations. For worst-case accuracy, higher is better, while for Avg-Bias@C, lower is better.}
     {\small
	\begin{tabular}{rcc}
	\toprule
    & WCA & Avg-Bias@C \\
		\midrule
		Reward model trained on $Z_C$ & $\mathbf{0.60}$ & $\mathbf{0.15}$\\
        Reward model trained on $Z_S$ & $0.53$ & $0.40$\\
        Reward model trained on $Z$ & $0.55$ & $0.42$\\
		\bottomrule
    \end{tabular}
    }
  \label{tab:ablation_study}
\end{table}

%% file: sections/6conclusion.tex
\section{Conclusion}\label{sec:conclusion}
We introduce a principled framework for mitigating spurious correlations in reward models. Our key insight is a shift in perspective: rather than training the reward model directly on observed data, we first isolate the bias-free latent variables that capture true human preferences. Theoretical guarantees and a practical algorithm are provided to achieve so. By training reward models exclusively on these bias-free latent variables, we obtain reward models that are more robust and better reflect human preferences, as demonstrated by experiments on synehetic and text-based datasets. This work establishes a foundation for principled bias mitigation in RLHF with theoretical guarantees.

%% file: sections/3related_works.tex
\section{Extended Discussion of Related Works}\label{app:related_works}
\subsection{Mitigating Bias in Reward Models}\label{app:related_works_rlhf}
Reward models are susceptible to learning spurious correlations~\citep{veitch2021counterfactual} from human preference data, leading to various issues including length, sycophancy, discrmination, and concept biases \citep{zhou2024explore}. To mitigate length bias, \citet{singhal2024long} proposed several strategies, such as  constructing length balanced preference datasets and applying explicit length penalties during training. \citet{shen2023loose} applied the Product-of-Experts (PoE) approach to separate human intent from response length during reward training, while \citet{chen2024odin} developed the ODIN method that explicitly disentangles the reward signal into quality and length components, discarding the length signal during policy fine-tuning to prevent the model from optimizing for verbosity. While these methods have demonstrated success in addressing their targeted issues, they generally tackle only a single bias, requiring researchers to identify specific spurious features and design tailored solutions. In contrast, our approach offers a principled framework capable of handling multiple types of biases.

To improve the robustness of reward models to distribution shifts, \citet{eisenstein2024helping} investigated reward model ensembles, finding that while they can reduce reward hacking, they do not eliminate it entirely because individual models often share systematic errors. As a computationally efficient alternative, \citet{rame2024warm} proposed Weight-Averaged Reward Models (WARM), which averages the weights of models from different training checkpoints to improve robustness against preference inconsistencies and distribution shifts. 

To handle different types of biases, other more general techniques include  \citet{wang2025beyond} that used a regularization method based on maximum mean discrepancy (MMD)~\citep{gretton2012kernel} to disentangle the true reward from biases. In a similar vein, \citet{srivastava2025robust} developed the Crome framework, a data-centric method that uses counterfactual augmentations to systematically teach the reward model to distinguish between meaningful quality improvements and superficial stylistic changes. While effective in certain scenarios, these methods typically lack rigorous theoretical guarantees for mitigating bias.

\subsection{Latent Variables Identification \& Causal Representation Learning}\label{app:related_works_crl}
Earlier works have shown that it is in general not possible to identify the underlying latent variables without additional assumptions~\citep{hyvarinen1999nonlinear,locatello2019challenging}. Even when the underlying latent variables are independent, one can construct infinitely many solutions that reproduce the observed distribution while keeping the true latent factors entangled~\citep{hyvarinen1999nonlinear,locatello2019challenging}. For the special case where the mapping from latent to observed variables is linear, one can impose non-Gaussianity assumption on the latent variables to achieve identifiability~\citep{comon1994independent,hyvarinen2002independent}, known as independent component analysis (ICA).

Beyond this linear setting, often referred to as nonlinear ICA, a major line of works leverage the assumption of sufficient variability across different data distributions, for instance through time or domain indices~\citep{hyvarinen2016unsupervised,hyvarinen2017nonlinear,hyvarinen2019nonlinear,khemakhem2020variational}. Alternatively, identifiability can be achieved by constraining the mixing function, such as restricting its function class~\citep{hyvarinen1999nonlinear,taleb1999source,gresele2021independent,buchholz2022function} or enforcing sparsity on it~\citep{zheng2022identifiability}.

In recent years, causal representation learning has moved beyond independent latent variables by focusing on causally-related ones~\citep{scholkopf2021causal,moran2025interpretable}. As in nonlinear ICA, a major line of works rely on the assumption that the distribution of latent variables varies sufficiently, e.g., through interventions~\citep{ahuja2023interventional,squires2023linear,vonkugelgen2023nonparametric,jiang2023learning,zhang2023identifiability,varici2023score,varici2024scorebased,varici2024linear,jin2023learning,bing2024identifying,zhang2024causal,ng2025causal}, temporal data~\citep{yao2022temporally,yao2022learning,lippe2022citris,lippe2023causal}, or a combination of both~\citep{lachapelle2021disentanglement,lachapelle2024nonparametric}. Other lines of works leverage structural assumptions~\citep{silva2006learning,xie2020generalized,cai2019triad,xie2022identification,adams2021identification,huang2022latent,dong2023versatile,kivva2021learning}, restriction on the support of latent variables~\citep{ahuja2023interventional,wang2021desiderata}, multi-view data~\citep{yao2024multiview,xu2024sparsity}, counterfactual view~\citep{brehmer2022weakly}, additional supervision~\citep{yang2021causalvae,shen2022weakly}, or prior knowledge of causal ordering~\citep{kori2023causal} or structure~\citep{liang2023causal}.

%% file: sections/7appendix.tex
\section{Proof of Identifiability Theory: With Access to Surrogate}\label{app:proof_identifiability_theory_with_known_spurious}
We first provide several useful lemmas in \cref{app:lemmas_for_identifiability_theory_with_known_spurious}, followed by the proof of \cref{theorem:identifiability_of_latent_variables_with_known_spurious} in \cref{app:main_proof_identifiability_theory_with_known_spurious}. The proof builds on \citet[Theorem~4.2]{kong2022partial} but incorporates several generalizations; see \cref{sec:identifiability_theory_with_known_spurious} for a discussion.

\subsection{Technical Lemmas}\label{app:lemmas_for_identifiability_theory_with_known_spurious}
We state a few lemmas from \citet{kong2022partial,ng2025causal} that will be useful for our proofs.
\begin{lemma}[\citet{kong2022partial}]\label{lemma:equivalent_statements}
Let $f:\mathcal{Z}_C\times \mathcal{Z}_S\rightarrow \mathcal{Z}_C$ be a continuous function. The following three statements are equivalent:
    \begin{itemize}
        \item Statement (a): $f(Z_C, Z_S)$ does not depend on $Z_S$.
        \item Statement (b): For all $\hat{Z}_C \in \mathcal{Z}_C $, we have $ f^{-1} (\hat{Z}_C) = B_{Z_C} \times Z_s $ where $ B_{Z_C} \neq \emptyset $ and $ B_{Z_C} \subseteq \mathcal{Z}_C $. 
        \item Statement (c): For all $\hat{Z}_C \in \mathcal{Z}_C $ and $r>0$, we have $ f^{-1} ( \mathcal{B}_{r} (\hat{Z}_C) ) = B^{+}_{Z_C} \times Z_S$, where $ \mathcal{B}_{r} (\hat{Z}_C):= \{ \hat{Z}'_C \in \mathcal{Z}_C: \| \hat{Z}'_C - \hat{Z}_C \|^{2} < r \} $, $ B^{+}_{Z_C} \neq \emptyset $, and $ B^{+}_{Z_C} \subseteq \mathcal{Z}_{c} $.
    \end{itemize}
\end{lemma}

\begin{lemma}[\citet{kong2022partial}]\label{lemma:integral}
Consider the generative process in \cref{eq:data_generating_process_known_spurious}. Denote $\bar{\psi} = \hat{g}^{-1} \circ g: \mathcal{Z} \to \mathcal{Z}$, i.e., $\hat{Z}=\bar{\psi}(Z)$, and $\hat{Z}_C=\bar{\psi}_C(Z)=\bar{\psi}(Z)_{1:n_c}$, where $\bar{\psi}_C: \mathcal{Z} \to \mathcal{Z}_C$. Suppose $\hat{Z}_C\independent S$ and we model the same generative process with $(\hat{g},p_{\hat{Z}})$. Then, for any $ A_{\hat{Z}_C} \subseteq \mathcal{Z}_C $ and $S_1,S_2\in\mathcal{S}$, we have
    \[
        \int_{ Z \in \bar{\psi}_{c}^{-1} ( A_{ \hat{Z}_{C} } ) } p(Z\mid  S=S_1 ) \, dZ
        = \int_{ Z \in \bar{\psi}_{c}^{-1} ( A_{ \hat{Z}_{C} } ) } p(Z\mid S=S_2) \, dZ.
    \]
\end{lemma}

\begin{lemma}[{\citet[Lemma~5]{ng2025causal}}]\label{lemma:sub_diffeomorphism}
Let $\hat{Z}\in\mathcal{Z}$ be an invertible transformation of $Z\in\mathcal{Z}$. Suppose that there exists $\mathcal{I}\subseteq[n]$ such that each $\hat{Z}_i,i\in \mathcal{I}$ does not depend on $Z_j,j\not\in \mathcal{I}$. Then, $\hat{Z}_{ \mathcal{I}}$ is an invertible transformation of $Z_{\mathcal{I}}$.
\end{lemma}

\subsection{Proof of \cref{theorem:identifiability_of_latent_variables_with_known_spurious}}\label{app:main_proof_identifiability_theory_with_known_spurious}

\TheoremIdentifiabilityLatentVariablesWithKnownSpurious*
\begin{proof}
Denote $\bar{\psi} = \hat{g}^{-1} \circ g: \mathcal{Z} \to \mathcal{Z}$, i.e., $\hat{Z}=\bar{\psi}(Z)$, and $\hat{Z}_C=\bar{\psi}_C(Z)=\bar{\psi}(Z)_{1:n_c}$, where $\bar{\psi}_C: \mathcal{Z} \to \mathcal{Z}_C$. Further denote $Z=(Z_C,Z_S)$. \cref{lemma:integral} implies that, for any $A_{\hat{Z}_C} \subseteq \mathcal{Z}_C$ and $S_1,S_2\in\mathcal{S}$, we have
    \begin{equation*}
        \int_{ Z \in \bar{\psi}_{c}^{-1} ( A_{ \hat{Z}_{C} } ) } p(Z\mid  S=S_1 ) \, dZ
        = \int_{ Z \in \bar{\psi}_{c}^{-1} ( A_{ \hat{Z}_{C} } ) } p(Z\mid S=S_2) \, dZ.
        \label{eq:equal_preimage_zc}
    \end{equation*}
By assumption of the generative process, we have $Z_C\independent S$. Therefore, the above equation can be written as
    \begin{equation}
            \int_{ (Z_C,Z_S) \in \bar{\psi}_{c}^{-1} ( A_{ \hat{Z}_{C} } ) } p (Z_C) \left(
                p(Z_S \mid Z_C, S=S_1) - p(Z_S \mid Z_C, S=S_2)
            \right) dZ_S dZ_C = 0. \label{eq:equation_to_test}
    \end{equation}

Our goal is to prove that $\hat{Z}_C=\bar{\psi}_{c}(Z_C,Z_S) $ does not depend on $Z_S$. By \cref{lemma:equivalent_statements}, it suffices to prove the following statement.

\begin{statement}\label{statement:existence_point}
For all $\hat{Z}_C \in \mathcal{Z}_C $ and $r>0$, we have $ \bar{\psi}_{c}^{-1} ( \mathcal{B}_{r} (\hat{Z}_C) ) = B^{+}_{Z_C} \times Z_S$, where $ \mathcal{B}_{r} (\hat{Z}_C):= \{ \hat{Z}'_C \in \mathcal{Z}_C: \| \hat{Z}'_C - \hat{Z}_C \|^{2} < r \} $, $ B^{+}_{Z_C} \neq \emptyset $, and $ B^{+}_{Z_C} \subseteq \mathcal{Z}_{c} $.
\end{statement}
We now prove the statement above by contradiction. Suppose by contradiction that \cref{statement:existence_point} does not hold, i.e., there exists $ \hat{Z}^{*}_C \in \mathcal{Z}_{c} $ and $r^*>0$ such that $B^{*}_{Z}:= \{Z=(Z_C,Z_S) \in \bar{\psi}_{c}^{-1} ( A_{\hat{Z}_{C}}^{*} ): Z_C \times \mathcal{Z}_{s} \not\subseteq \bar{\psi}_{c}^{-1} ( A_{\hat{Z}_{C}}^{*} ) \} \neq \emptyset$, where $A_{\hat{Z}_C}^{*}:= \mathcal{B}_{r^{*}} (\hat{Z}^{*}_C) $. We now compute the LHS of \cref{eq:equation_to_test} with $A_{\mathbf{z}_{c}}^{*}$. Since $\bar{\psi}_{c}^{-1} ( A_{\hat{Z}_{C}}^{*} )$ by definition, we have
    \begin{align*}
        &\int_{ (Z_C,Z_S) \in \bar{\psi}_{c}^{-1} ( A^*_{ \hat{Z}_{C} } ) } p (Z_C) \left(
                p(Z_S \mid Z_C, S=S_1) - p(Z_S \mid Z_C, S=S_2)
            \right) dZ_S dZ_C \\ 
        &= \int_{(Z_C,Z_S) \in \bar{\psi}_{c}^{-1} ( A^{*}_{ \hat{Z}_C } ) \setminus B^{*}_{Z} } p (Z_C) \left(
                p(Z_S \mid Z_C, S=S_1) - p(Z_S \mid Z_C, S=S_2)
            \right) dZ_S dZ_C
        \\
        &\qquad+ 
            \int_{ (Z_C,Z_S) \in B^{*}_{Z} } p (Z_C) \left(
                p(Z_S \mid Z_C, S=S_1) - p(Z_S \mid Z_C, S=S_2)
            \right) dZ_S dZ_C.
    \end{align*} 

    Denote by $T_1$ and $T_2$ the first and second terms in the RHS of the above equation, respectively. We first evaluate $T_1$.  By definition of $B_Z^*$, we can write $ \bar{\psi}_{c}^{-1} ( A^{*}_{\hat{Z}_C} ) \setminus B^{*}_{Z}= C^{*}_{Z_C} \times \mathcal{Z}_{S} $ for some $C^{*}_{Z_C}\subset \mathcal{Z}_C$, which can possibly be an empty set. This implies
    \begin{align*}
        T_1 & =\int_{ (Z_C,Z_S) \in  C^{*}_{Z_C} \times \mathcal{Z}_{S} } p (Z_C) \left(
                p(Z_S \mid Z_C, S=S_1) - p(Z_S \mid Z_C, S=S_2)
            \right) dZ_S dZ_C \\
        & = \int_{ Z_C \in C^{*}_{Z_C} }p (Z_C) \int_{ Z_S \in \mathcal{Z}_{s} } \left(
                p(Z_S \mid Z_C, S=S_1) - p(Z_S \mid Z_C, S=S_2)
            \right) dZ_S dZ_C \\
        & = \int_{ Z_C \in C^{*}_{Z_C} }p (Z_C) \left(\int_{ Z_S \in \mathcal{Z}_{s} }
                p(Z_S \mid Z_C, S=S_1)  dZ_S -\int_{ Z_S \in \mathcal{Z}_{s} }
                p(Z_S \mid Z_C, S=S_2)  dZ_S \right)dZ_C \\
        & = \int_{ Z_C \in C^{*}_{Z_C}} p (Z_C) \left( 1 - 1 \right) dZ_C = 0.
    \end{align*}

    After evaluating $T_1$, we now consider $T_2$. We first show that $B_Z^*$ has nonzero probability measure. First recall that $B_Z^*$ is assumed to be non-empty. By the continuity of $\bar{\psi}_{c}(\cdot)$, for all $Z' \in B_{Z}^{*} $, there exists $ r'>0$ such that $ \mathcal{B}_{r'} (Z') \subseteq B_{Z}^{*} $.
    Since $p(Z\mid S) > 0$ for $Z\in\mathcal{Z}$ and $S\in\mathcal{S}$, it follows that $ \Pb{ \{ Z \in B^{*}_{Z} \} \mid \{ S = S' \} } \ge \Pb{ \{ Z \in \mathcal{B}_{r'} (Z') \mid \{ S = S' \} \} } > 0 $ for $ S' \in \mathcal{S}$, indicating that $B_Z^*$ has nonzero probability measure

    By Assumption A4, there exist $S_{1}^{*}$ and $S_{2}^{*}$ such that
    \begin{align*}
        T_{2} = \int_{ (Z_C,Z_S) \in B^{*}_Z } p(Z_C) \left(
            p(Z_S \mid Z_C, S=S_{1}^{*} ) - p( Z_S \mid Z_C, S=S_{2}^{*} )
        \right) dZ_S dZ_C \neq 0.
    \end{align*}
    Clearly, we then have $ T_{1} + T_{2} \neq 0 $, which is contradictory with \cref{eq:equation_to_test}, because the latter holds for all $S_1$ and $S_2$. By contradiction, this implies that \cref{statement:existence_point} must be true, and therefore $ \bar{\psi}_{c} (\cdot) $ cannot depend on $ Z_S$.

    Since $\hat{Z}_C$ does not depend on $Z_S$, by setting $\hat{Z}_{\mathcal{I}}=\hat{Z}_C$ and $Z_{\mathcal{I}}=Z_C$ in \cref{lemma:sub_diffeomorphism}, we conclude that $\hat{Z}_C$ is an invertible transformation of $Z_C$, indicating that $Z_C$ is subspace identifiable.
\end{proof}

\section{Proof of Identifiability Theory: Without Access to Surrogate}\label{app:proof_identifiability_theory_with_unknown_spurious}
We first state several key lemmas in \cref{app:lemmas_for_identifiability_theory_with_unknown_spurious} and then provide the proof of \cref{theorem:identifiability_of_shared_latent_variables} in \cref{app:main_proof_identifiability_of_shared_latent_variables}. The proof is inspired by \citet[Theorem 3.1]{lachapelle2022partial}, originally developed for multi-task learning; see \cref{sec:identifiability_theory_with_unknown_spurious} for a detailed discussion.

\subsection{Technical Lemmas}\label{app:lemmas_for_identifiability_theory_with_unknown_spurious}
We present several results from \citet{lachapelle2023synergies} that will be useful in our proofs. In particular, \cref{lemma:relation_between_support} is an intermediate result derived from \citet[Theorem~B.5]{lachapelle2023synergies}, while \cref{lemma:linear_representation} is stated almost as in \citet[Theorem~B.4]{lachapelle2023synergies}, with only minor modifications in notation.

\begin{lemma}[Linear identifiability {\citep[Theorem~B.4]{lachapelle2023synergies}}]\label{lemma:linear_representation}
Let $\hat{W}^{(\cdot)}: \mathcal{W} \rightarrow \mathbb{R}^{k \times n}$. Suppose Assumptions A2, A3, and A4 from \cref{theorem:identifiability_of_shared_latent_variables} hold, and that, for $\mathbb{P}_{W}$-almost every $W \in \mathcal{W}$ and all $T \in \mathcal{T}$, the following holds
\[
\KL(p(R ; \hat{W}^{(W)} \hat{f}(T)) \parallel p(R ; W f(T))=0.
\]
Then, there exists an invertible matrix $L \in \mathbb{R}^{n \times n}$ such that, for all $T \in \mathcal{T}, f(T)=L \hat{f}(T)$ and such that, for $\mathbb{P}_W$-almost every $W \in \mathcal{W}, \hat{W}^{(W)}=W L$.
\end{lemma}

\begin{lemma}[\citet{lachapelle2023synergies}]\label{lemma:relation_between_support}
Suppose $WL=\hat{W}^{(W)}$ for $W\in\mathcal{W}$, where $L\in\mathbb{R}^{n\times n}$ is invertible, and that $\mathbb{E}_{\mathbb{P}_W}[\|\hat{W}\|_0]\leq \mathbb{E}_{\mathbb{P}_W}[\|W\|_0]$. Denote by $A^{(W)}$ the support of $W$, and $N_{i}$ the support of $L_{\cdot,i}$. Then, under Assumption A5 from \cref{theorem:identifiability_of_shared_latent_variables}, there exists a permutation $\sigma:[n]\rightarrow [n]$ such that $L_{i,\sigma(i)}\neq 0,i\in[n]$, and for $j\not\in A^{(W)}$, we have $A^{(W)}\cap N_{\sigma_j}=\emptyset$.
\end{lemma}

\subsection{Proof of \cref{theorem:identifiability_of_shared_latent_variables}}\label{app:main_proof_identifiability_of_shared_latent_variables}
The derivation from \cref{eq:proof:first_kl_term} to \cref{eq:proof:smaller_W_norm} in the proof below is adapted from \citet[Theorem~B.6]{lachapelle2023synergies}, which is a standard approach for showing that minimizing the negative log-likelihood under constraints produces a solution with zero KL divergence from the true distribution.

\TheoremIdentifiabilitySharedLatentVariables*

\begin{proof}
To lighten the notation, recall that each $W\in\mathcal{W}$ is a row vector. Define $\mathrm{supp}(W)\coloneqq \{i:W_{1,i}\neq 0\}$ as the support of row vector $W$. Let
\[
B\coloneqq\bigcap_{A \in \mathcal{A}} A=\bigcap_{W \in \mathcal{W}} \mathrm{supp}(W)
\]
and $\hat{f}$ be a minimizer of
\begin{equation}\label{eq:optimization_for_theory}
\begin{aligned}
\min_{\hat{f}} \quad & \mathbb{E}_{\mathbb{P}_W}\mathbb{E}_{p(X,R\mid W)} [-\log p(R;\hat{W}^{(W)} \hat{f}(T))]\\
\subjectto \quad & \hat{W}^{(W)} \in \argmin_{\tilde{W} \subjectto \|\tilde{W}\|_0\leq \|W\|_0} \mathbb{E}_{p(X,R\mid W)} [-\log p(R;\tilde{W} \hat{f}(T))]\\
\mathrm{and} \quad & |\hat{B}|\leq |B|.\\
\mathrm{where} \quad & \hat{B}=\bigcap_{W \in \mathcal{W}} \mathrm{supp}(\hat{W}^{(W)}).
\end{aligned}
\end{equation}

Note that
\begin{flalign}
\mathbb{E}_{\mathbb{P}_{W}} \mathbb{E}_{p(X \mid W)} [\KL(p(R ; W f(T)) \parallel p(R ; \hat{W}^{(W)} \hat{f}(T)))] &\geq 0 \label{eq:proof:first_kl_term}\\
\mathbb{E}_{\mathbb{P}_{W}} \mathbb{E}_{p(X, R \mid W)}[-\log p(R ; \hat{W}^{(W)} \hat{f}(T))]
 &\geq \mathbb{E}_{\mathbb{P}_{W}} \mathbb{E}_{p(X, R \mid W)}[-\log p(R ; W f(T))].\nonumber
\end{flalign}
Therefore, assuming without constraint, the objective in \cref{eq:optimization_for_theory} is minimized if and only if the equalities are attained in the above inequalies. This implies
\[
\mathbb{E}_{p(X \mid W)} [\KL(p(R ; W f(T)) \parallel p(R ; \hat{W}^{(W)} \hat{f}(T)))]=0
\]
$\mathbb{P}_{W}$-almost everywhere. For a given $W$, the above equality holds if and only if
\begin{equation}\label{eq:proof_KL_divergence_zero}
\KL(p(R ; W f(T)) \parallel p(R ; \hat{W}^{(W)} \hat{f}(T)))=0
\end{equation}
$p(X \mid W)$-almost everywhere.

Recall that the zero KL term above is obtained under the assumption that \cref{eq:optimization_for_theory} is solved without constraints. To show that the same minimum can be attained under the constrained problem, it suffices to show the existence of a solution that achieves this under the constraints. Such a solution can be constructed by setting $\hat{f}\coloneqq f \quad\textrm{and}\quad \hat{W}^{(W)}\coloneqq W$. Clearly, this choice satisfies the same objective value as the unconstrained global minimum. It is straightforward to verify that the first constraint holds because $\|\hat{W}^{(W)}\|_0= \|W\|_0$ and
\[
\begin{aligned}
\mathbb{E}_{p(X \mid W)} [\KL(p(R ; W f(T)) \parallel p(R ; \tilde{W} f(T)))] & \geq 0\\
\mathbb{E}_{p(X, R \mid W)}[-\log p(R ; \tilde{W} f(T))] &\geq \mathbb{E}_{p(X, R \mid W)}[-\log p(R ; W f(T))],
\end{aligned}
\]
with equality attained when $\tilde{W}:=W$. The remaining constraint holds trivially because $|\hat{B}|=|B|$.

We have shown that any minimizer $\hat{f}$ of \cref{eq:optimization_for_theory} satisfies \cref{eq:proof_KL_divergence_zero} and
\begin{flalign}
\|\hat{W}^{(W)}\|_0 &\leq \|W\|_0,\quad W\in\mathcal{W},\label{eq:proof:smaller_W_norm}\\
|\hat{B}| &\leq |B|.\label{eq:proof:smaller_B_size}
\end{flalign}
The first inequality above implies
\begin{equation}\label{eq:minimal_expected_sparsity}
\mathbb{E}_{\mathbb{P}_{W}}[\|\hat{W}^{(W)}\|_0] \leq \mathbb{E}_{\mathbb{P}_{W}}[\|W\|_0],
\end{equation}
By Assumptions A2, A3, and A4, as well as \cref{eq:proof_KL_divergence_zero}, \cref{lemma:linear_representation} implies that there exists an invertible matrix $L \in \mathbb{R}^{n \times n}$ such that, for all $T \in \mathcal{T}$, we have
\begin{equation}\label{eq:proof_linear_representation_f}
f(T)=L \hat{f}(T),
\end{equation}
and such that, for $\mathbb{P}_W$-almost every $W \in \mathcal{W}$, we have
\begin{equation}\label{eq:proof_linear_representation_W}
\hat{W}^{(W)}=W L.
\end{equation}

With a slight abuse of notation, denote by $A^{(W)}$ the support of $W$. Also, let $N_{i}$ be the support of $L_{\cdot,i}$. By Assumption A7, as well as \cref{eq:proof:smaller_W_norm,eq:proof_linear_representation_W}, \cref{lemma:relation_between_support} implies that there exists a permutation $\sigma:[n]\rightarrow [n]$ such that $L_{i,\sigma(i)}\neq 0,i\in[n]$, and for $j\not\in A^{(W)}$, we have  $A^{(W)}\cap N_{\sigma_j}=\emptyset$. This indicates that:
\begin{equation}\label{eq:proof_relation_support}
\textrm{for~~} A\in\mathcal{A} \textrm{~~and~~} j\not\in A, \quad\textrm{we have~~} A\subset (N_{\sigma_j})^c,
\end{equation}
Since this holds for all $j\not\in A$, we have 
\[
l\in A \implies l\not\in \bigcup_{j\not\in A} N_{\sigma(j)}.
\]
Since \cref{eq:proof_relation_support} holds for all $A\in\mathcal{A}$ and $j\not\in A$, we have
\[
l\in \bigcap_{A\in\mathcal{A}}A \implies l\not\in \bigcup_{A\in\mathcal{A}}\bigcup_{j\not\in A} N_{\sigma(j)},
\]
i.e.,
\[
l\in B \implies l\not\in \bigcup_{j\not\in B} N_{\sigma(j)}.
\]
Recall that, by definition, we have $\hat{Z}=\hat{f}(T)$ and $Z=f(T)$. Furthermore, by \cref{eq:proof_linear_representation_f} and definition of $N_{\sigma(j)}$, $l\not\in N_{\sigma(j)}$ indicates that $Z_l$ does not depend on $\hat{Z}_{\sigma(j)}$. The above statement then implies that each $Z_l,l\in B$ does not depend on $\hat{Z}_{\sigma(j)},j\not\in B$. Since matrix $L$ is invertible by definition, \cref{lemma:sub_diffeomorphism} implies that $Z_B$ is an invertible transformation of $\hat{Z}_{\sigma(B)}$, where $\sigma(B)\coloneqq \{\sigma(i):i\in B\}$. 

It remains to show $\hat{Z}_{\hat{B}}=\hat{Z}_{\sigma(B)}$, i.e., $\hat{B}=\sigma(B)$. Suppose $l\in B$. Since $L_{l,\sigma(l)}\neq 0$ by definition, we have
\[
(\hat{W}^{(W)})_{1,\sigma(l)}=(WL)_{1,\sigma(l)}=WL_{\cdot,\sigma(l)}=W_{\cdot, A^{(W)}}L_{A^{(W)},\sigma(l)},
\]
where, as before, $A^{(W)}$ denotes the support of $W$. Moreover, we have $l\in B\subseteq A^{(W)}$, indicating that $L_{A^{(W)},\sigma(l)}$ cannot be a zero vector. By Assumption A7, we have $(\hat{W}^{(W)})_{1,\sigma(l)}\neq 0$ almost surely, which implies $\sigma(l)\in \hat{B}$ by definition of $\hat{B}$. Since we have established that $l\in B\implies \sigma(l)\in \hat{B}$ almost surely, it follows that
\begin{equation}\label{eq:proof:b_subset}
\sigma(B)\subseteq \hat{B}
\end{equation}
almost surely. By \cref{eq:proof:smaller_B_size}, we have
\[
|\hat{B}| \leq |B|=|\sigma(B)|,
\]
which, with \cref{eq:proof:b_subset}, implies
\[
\hat{B}=\sigma(B)
\]
almost surely. Recall that we have shown that $\hat{Z}_{\sigma(B)}$ is an invertible transformation of $Z_B$. Therefore, $\hat{Z}_{\hat{B}}$ is an invertible transformation of $Z_B$ almost surely, indicating that $Z_B$ is subspace identifiable almost surely.
\end{proof}

\section{Proof of \cref{corollary:identifiability_of_Zc_with_unknown_spurious}}\label{app:proof_corollary_identifiability_of_Zc_with_unknown_spurious}

\CorollaryIdentifiabilityLatentVariablesWithUnknownSpurious*
\begin{proof}
Denote
\[
B\coloneqq\bigcap_{A \in \mathcal{A}} A=\bigcap_{W \in \mathcal{W}} \mathrm{supp}(W)\quad\textrm{and}\quad \hat{B}=\bigcap_{W \in \mathcal{W}} \mathrm{supp}(\hat{W}^{(W)}).
\]
Under Assumptions A2, A3, A4, and A5 from \cref{theorem:identifiability_of_shared_latent_variables}, \cref{theorem:identifiability_of_shared_latent_variables} implies that $\hat{Z}_{\hat{B}}$ is an invertible transformation of $Z_B$ almost surely. By the assumption that $Z_C \subseteq Z_A$ for all $A \in \mathcal{A}$, it follows that $Z_C \subseteq Z_B$.\footnote{When $Z_C$ and $Z_B$ are vectors, we interpret $\subseteq$ as inclusion of their components, i.e., $Z_C \subseteq Z_B$ means that every component of $Z_C$ is also a component of $Z_B$.} Moreover, the assumption $\bigcap_{A \in \mathcal{A}} (Z_A \setminus Z_C) = \emptyset$ implies $Z_B \subseteq Z_C$. Combining these, we obtain $Z_B = Z_C$, and therefore $\hat{Z}_{\hat{B}}$ is an invertible transformation of $Z_C$ almost surely.
\end{proof}

\section{Supplementary Experiment Details}\label{app:experiment_details}
\paragraph{Datasets.}
For the experiments involving sycophancy bias in \cref{sec:experiments_sycophancy_bias}, we use a dataset derived from the QA preference dataset \citep{sharma2024towards}, which in turn is based on the TruthfulQA dataset~\citep{lin2021truthfulqa}. For the experiments involving concept bias in \cref{sec:experiments_concept_bias}, we use a modified version of the Amazon Shoe Review dataset from \citet{zhou2024explore}, originally derived from the broader Amazon dataset~\citep{hou2024bridging}. The two datasets are available under the Apache 2.0 License and the MIT License, respectively.

\paragraph{Implementation of Our Approach.}
For Stage 1, our method for learning the bias-free latent variables $\hat{Z}_C$ is based on the LangVAE framework~\citep{carvalho2025langvae}. Our key modification involves replacing its independent prior with our prior $p(\hat{Z}\mid S)$, where the normalizing flow $\tau_i$ is implemented using deep sigmoidal flow. We train the VAE with a batch size of $240$ and set the hyperparameters in the loss function to $\beta=0.1$ and $\lambda=3$. For Stage 2, the reward model is then trained on $\hat{Z}_C$ using the standard implementation from the OpenRLHF library~\citep{hu2024openrlhf} involving a $2$-layer MLP with $128$ Leaky-ReLU units reward model. 

\paragraph{Implementation of Baselines.}
The vanilla reward model and CRM use a similar implementation of reward model as our approach. The key difference is that they are trained directly on BERT text embeddings instead of our learned bias-free representations $\hat{Z}_C$. For the CRM baseline, we treat it favorably by selecting the MMD regularization coefficient from $\{1, 3, 10, 30, 100\}$ that leads to the best performance. This resulted in coefficients of $100$ and $10$ for the experiments in \cref{sec:experiments_sycophancy_bias,sec:experiments_concept_bias}, respectively.

\paragraph{Others.}
All experiments are conducted on an NVIDIA A10G GPU with 4 CPU cores. The reported results are averaged over 8 random trials.